\documentclass{article}

\usepackage[preprint]{nips_2018}

\usepackage[utf8]{inputenc} % allow utf-8 input
\usepackage[T1]{fontenc}    % use 8-bit T1 fonts
\usepackage{hyperref}       % hyperlinks
\usepackage{url}            % simple URL typesetting
\usepackage{booktabs}       % professional-quality tables
\usepackage{amsfonts}       % blackboard math symbols
\usepackage{nicefrac}       % compact symbols for 1/2, etc.
\usepackage{microtype}      % microtypography

\usepackage{graphicx}
\usepackage{subfigure}

\usepackage{wrapfig}
\usepackage{natbib}
\usepackage{algorithm}
\usepackage{algorithmic}
\usepackage{color}
\usepackage{amsmath}
\usepackage{amssymb}

\usepackage{amsthm,amssymb}

\newcommand{\norm}[1]{\left\lVert#1\right\rVert}

\newtheorem{theorem}{Theorem}[section]

\newtheorem{lemma}[theorem]{Lemma}
\newtheorem{definition}{Definition}
\newtheorem{proposition}{Proposition}
\newtheorem{remark}{Remark}

\def\FIM{\texttt{FIM}}
\def\RR{\mathbb{R}}
\def\DD{\mathbb{D}}
\def\NN{\mathbb{N}}
\def\PP{\mathbb{P}}
\def\SS{\mathbb{S}}
\def\EE{\mathbb{E}}

\def\Hh{\mathcal{H}}

\def\Pp{\mathcal{P}}

\def\Ww{\mathcal{W}}
\def\Ff{\mathcal{F}}

\def\Nn{\mathcal{N}}

\def\PF{\text{PF}}

\def\PSS{\text{PSS}}
\def\PWG{\text{PWG}}
\def\SW{\text{SW}}

\def\pers{\text{pers}}

\def\Dg{\text{Dg}}

\renewcommand{\algorithmicrequire}{\textbf{Input:}}
\renewcommand{\algorithmicensure}{\textbf{Output:}}

\DeclareMathOperator*{\argmax}{arg\,max}

\title{Persistence Fisher Kernel: A Riemannian Manifold Kernel for Persistence Diagrams}

\author{
   Tam Le \\ %\thanks{Use footnote for providing further
  RIKEN Center for Advanced Intelligence Project, Japan\\
  \texttt{tam.le@riken.jp} \\
  \And
  Makoto Yamada \\ %\thanks{Kyoto University, Japan} \\
  Kyoto University, Japan \\
  RIKEN Center for Advanced Intelligence Project, Japan \\
  \texttt{makoto.yamada@riken.jp} \\
}

\begin{document}

\maketitle

\begin{abstract}
Algebraic topology methods have recently played an important role for statistical analysis with complicated geometric structured data such as shapes, linked twist maps, and material data. Among them, \textit{persistent homology} is a well-known tool to extract robust topological features, and outputs as \textit{persistence diagrams} (PDs). However, PDs are point multi-sets which can not be used in machine learning algorithms for vector data. To deal with it, an emerged approach is to use kernel methods, and an appropriate geometry for PDs is an important factor to measure the similarity of PDs. A popular geometry for PDs is the \textit{Wasserstein metric}. However, Wasserstein distance is not \textit{negative definite}. Thus, it is limited to build positive definite kernels upon the Wasserstein distance \textit{without approximation}. In this work, we rely upon the alternative \textit{Fisher information geometry} to propose a positive definite kernel for PDs \textit{without approximation}, namely the Persistence Fisher (PF) kernel. Then, we analyze eigensystem of the integral operator induced by the proposed kernel for kernel machines. Based on that, we derive generalization error bounds via covering numbers and Rademacher averages for kernel machines with the PF kernel. Additionally, we show some nice properties such as stability and infinite divisibility for the proposed kernel. Furthermore, we also propose a linear time complexity over the number of points in PDs for an approximation of our proposed kernel with a bounded error. Throughout experiments with many different tasks on various benchmark datasets, we illustrate that the PF kernel compares favorably with other baseline kernels for PDs.
\end{abstract}

%%%%%%%%%%%%%%%%%%%%
%%%%%%%%%%%%%%%%%%%%
\section{Introduction}

Using algebraic topology methods for statistical data analysis has been recently received a lot of attention from machine learning community \citep{chazal2015subsampling, kwitt2015statistical, bubenik2015statistical, kusano2016persistence, chen2016clustering, carriere17asliced, hofer2017deep, adams2017persistence, kusano2017kernelJMLR}. 
Algebraic topology methods can produce a robust descriptor which can give useful insight when one deals with complicated geometric structured data such as shapes, linked twist maps, and material data. More specifically, algebraic topology methods are applied in various research fields such as biology \citep{kasson2007persistent, xia2014persistent, cang2015topological}, brain science \citep{singh2008topological, lee2011discriminative, petri2014homological}, and information science \citep{de2007coverage, carlsson2008local}, to name a few.

In algebraic topology, \textit{persistent homology} is an important method to extract robust topological information, it outputs point multisets, called \textit{persistence diagrams} (PDs) \citep{edelsbrunner2000topological}. Since PDs can have different number of points, it is not straightforward to plug PDs into traditional statistical machine learning algorithms, which often assume a vector representation for data. 

\paragraph{Related work.} There are two main approaches in topological data analysis: (i) explicit vector representation for PDs such as computing and sampling functions built from PDs (i.e. persistence lanscapes \citep{bubenik2015statistical}, tangent vectors from the mean of the square-root framework with principal geodesic analysis \citep{anirudh2016riemannian}, or persistence images \citep{adams2017persistence}), using points in PDs as roots of a complex polynomial for concatenated-coefficient vector representations \citep{di2015comparing}, or using distance matrices of points in PDs for sorted-entry vector representations \citep{carriere2015stable}, (ii) implicit representation via kernels such as the Persistence Scale Space (PSS) kernel, motivated by a heat diffusion problem with a Dirichlet boundary condition \citep{reininghaus2015stable}, the Persistence Weighted Gaussian (PWG) kernel via kernel mean embedding \citep{kusano2016persistence}, or the Sliced Wasserstein (SW) kernel under Wasserstein geometry \citep{carriere17asliced}. In particular, geometry on PDs plays an important role. One of the most popular geometries for PDs is the Wasserstein metric \citep{villani2003topics, PeyreCuturiBook}. However, it is well-known that the Wasserstein distance is not \textit{negative definite} \citep{reininghaus2015stable} (Appendix A). Consequently, we may not obtain positive definite kernels, built upon from the Wasserstein distance. Thus, it may be necessary to \textit{approximate} the Wasserstein distance to achieve positive definiteness for kernels, relied on Wasserstein geometry. For example, \citep{carriere17asliced} used the SW distance---an \textit{approximation} of Wasserstein distance---to construct the positive definite SW kernel.

\paragraph{Contributions.} In this work, we focus on the implicit representation via kernels for PDs approach, and follow \cite{anirudh2016riemannian} to explore an alternative Riemannian geometry, namely the Fisher information metric \citep{amari2007methods, lee2006riemannian} for PDs. Our contribution is two-fold: (i) we propose a positive definite kernel, namely the Persistence Fisher (PF) kernel for PDs. The proposed kernel well preserves the geometry of the Riemannian manifold since it is directly built upon the Fisher information metric for PDs \textit{without approximation}. (ii) We analyze the eigensystem of the integral operator induced by the PF kernel for kernel machines. Based on that, we derive generalization error bounds via covering numbers and Rademacher averages for kernel machines with the PF kernel. Additionally, we provide some nice properties such as a bound for the proposed kernel induced squared distance with respect to the geodesic distance which can be interpreted as stability in a similar sense as the work of \citep{kwitt2015statistical, reininghaus2015stable} with Wasserstein geometry, and infinite divisibility for the proposed kernel. Furthermore, we describe a linear time complexity over the number of points in PDs for an approximation of the PF kernel with a bounded error via Fast Gauss Transform \citep{greengard1991fast, morariu2009automatic}.  

%The paper is organized as follows: we give a preliminary about PDs with the Wasserstein geometry and the Riemannian geometry, as well as some important theorems about kernels in Section \ref{sec:Background}. In Section \ref{sec:RMK}, we describe the proposed Persistence Fisher kernel and prove that it is positive definite. Then, we provide theoretical analysis about the eigensystem of the integral operator induced with our proposed kernel, and generalization error bounds on covering numbers and Rademacher averages, as well as bounds on the proposed kernel induced squared distance with respect to the geodesic distance of the Riemannian manifold between PDs and infinite divisibility of our proposed kernel in Section \ref{sec:Theoretical}. After that, we review related work in Section \ref{sec:Related}. We then illustrate that our proposed kernel improves performances of other baseline kernels through experiments with many different learning tasks on various benchmark datasets in Section \ref{sec:Experimental}. Finally, we conclude it in Section \ref{sec:Conclusions}.

%%%%%%%%%%%%%%%%%%%%
%%%%%%%%%%%%%%%%%%%%
\section{Background}\label{sec:Background}

\begin{figure}
%  \vspace{-40pt}
  \begin{center}
    \includegraphics[width=0.4\textwidth]{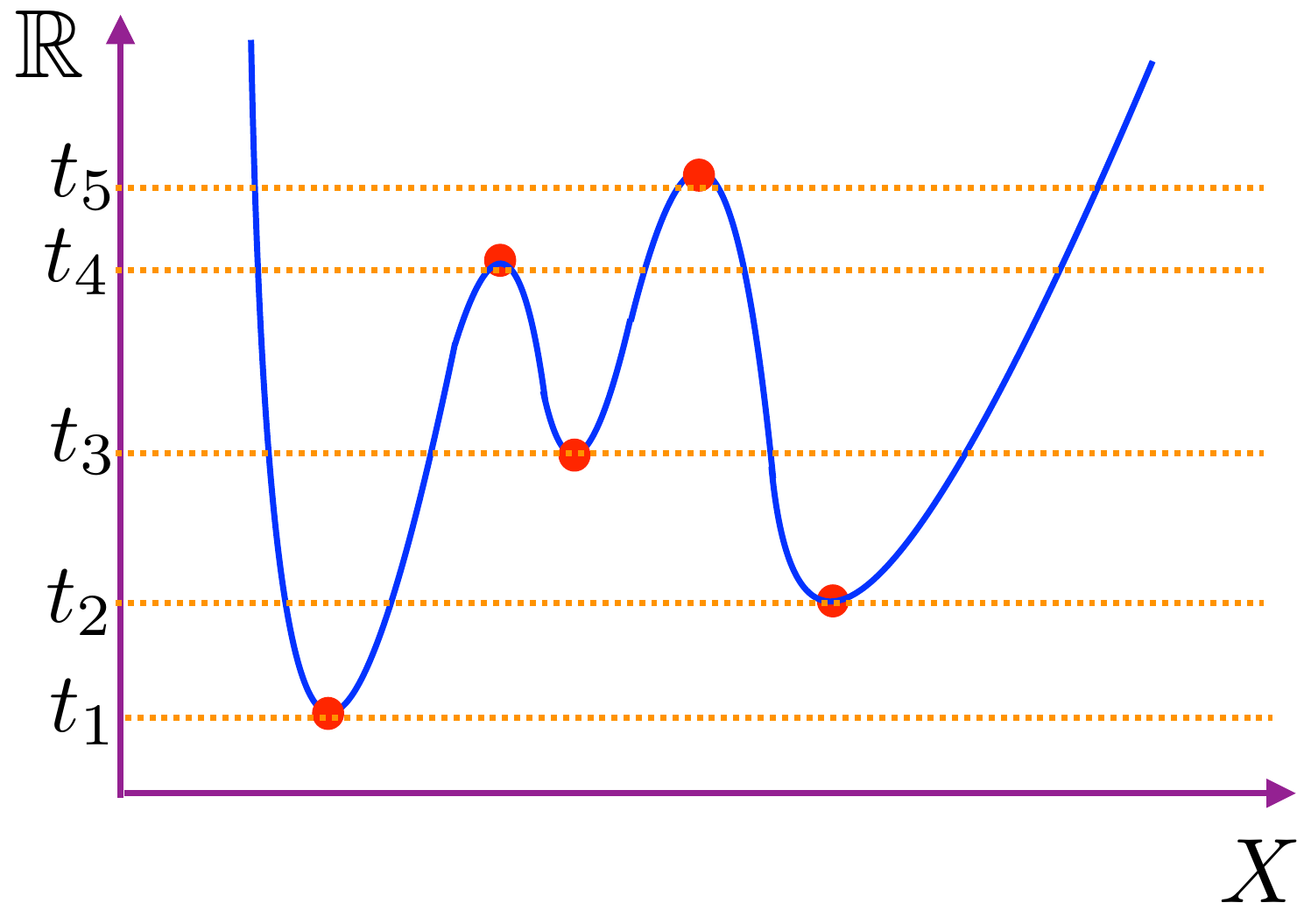}
  \end{center}
  \vspace{-15pt}
  \caption{An illustration of a persistence diagram on a real-value function $f$. The orange horizontal lines are the boundaries of sublevel sets $f^{-1}\!\left((-\infty, t]\right)$. For the $0$-dimensional topological features (connected components), the topological events of births are happened at $t = t_1, t_2, t_3$ and their corresponding topological events of deaths are occurred at $t = +\infty, t_5, t_4$ respectively. Therefore, the persistent diagram of $f$ is $\Dg f = \left\{(t_1, +\infty), (t_2, t_5), (t_3, t_4) \right\}$.}
  \label{fg:Persistence}
  \vspace{-15pt}
\end{figure}

%%%%%%%%%%%%%%%%%%%%
\paragraph{Persistence diagrams.} Persistence homology (PH) \citep{edelsbrunner2008persistent} is a popular technique to extract robust topological features (i.e. connected components, rings, cavities) on real-value functions. Given $f: X \mapsto \RR$, PH considers the family of \textit{sublevel sets} of $f$ (i.e. $f^{-1}((-\infty, t]), t\in \RR$) and records all \textit{topological events} (i.e. births and deaths of topological features) in $f^{-1}((-\infty, t])$ when $t$ goes from $-\infty$ to $+\infty$. PH outputs a 2-dimensional point multiset, called \textit{persistence diagram} (PD), illustrated in Figure \ref{fg:Persistence}, where each 2-dimensional point represents a lifespan of a particular topological feature with its birth and death time as its coordinates. 

%%%%%%%%%%%%%%%%%%%%
\paragraph{Wasserstein geometry.} Persistence diagram $\Dg$ can be considered as a discrete measure $\mu_{\Dg} = \sum_{u \in \Dg} \delta_u$ where $\delta_u$ is the Dirac unit mass on $u$. Therefore, the bottleneck metric (a.k.a. $\infty$-Wasserstein metric) is a popular choice to measure distances on the set of PDs with bounded cardinalities. Given two PDs $\Dg_{i}$ and $\Dg_{j}$, the bottleneck distance $\Ww_{\infty}$ \citep{cohen2007stability, carriere17asliced, adams2017persistence} is defined as 
\[
\Ww_{\infty}\left(\Dg_{i}, \Dg_{j} \right) = \inf_{\gamma} \sup_{x \in \Dg_i \cup \Delta} \norm{x - \gamma(x)}_{\infty},
\]
where $\Delta := \left\{ \left(a, a\right) \mid a \in \RR\right\}$ is the diagonal set, and $\gamma : \Dg_{i} \cup \Delta \rightarrow \Dg_{j} \cup \Delta$ is bijective. 

%%%%%%%%%%%%%%%%%%%%
\paragraph{Fisher information geometry.} Given a bandwidth $\sigma > 0$, for a set $\Theta$, one can smooth and normalize $\mu_{\Dg}$ as 
follows, %a probability distribution \citep{anirudh2016riemannian, adams2017persistence}
\begin{equation}\label{equ:NSprob}
\rho_{\Dg} := \left[ \frac{1}{Z} \sum_{u \in \Dg} \mathtt{N}(x; u, \sigma I) \right]_{x \in \Theta}, 
\end{equation}
where $Z = \int_{\Theta} \sum_{u \in \Dg} \mathtt{N}(x; u, \sigma I) \mathrm{d}x $, $\mathtt{N}$ is a Gaussian function and $I$ is an identity matrix. Therefore, each PD can be regarded as a point in a probability simplex $\PP := \left\{ \rho \mid \int \rho(x) \mathrm{d}x = 1, \rho(x) \ge 0\right\}$\footnote{In case, $\Theta$ is an \textit{infinite} set, then the corresponding probability simplex $\PP$ has \textit{infinite} dimensions.}. In case, one chooses $\Theta$ as an entire Euclidean space, each PD turns into a probability distribution as in \citep{anirudh2016riemannian, adams2017persistence}.

Fisher information metric (FIM)\footnote{FIM is also known as a particular pull-back metric on Riemannian manifold \citep{le2015unsupervised}.} is a well-known Riemannian geometry on the probability simplex $\PP$, especially in information geometry \citep{amari2007methods}. Given two points $\rho_i$ and $\rho_j$ in $\PP$, the Fisher information metric is defined as
\begin{equation}\label{equ:dFIM}
d_{\Pp}(\rho_i, \rho_j) = \arccos \! \left(\int \!\!\sqrt{\rho_i(x) \rho_j(x)} \mathrm{d}x \right).
\end{equation}

%%%%%%%%%%%%%%%%%%%%
%%%%%%%%%%%%%%%%%%%%
\section{Persistence Fisher Kernel (PF Kernel)}\label{sec:RMK}

In this section, we propose the Persistence Fisher (PK) kernel for persistence diagrams (PDs).

For the bottleneck distance, two PDs $\Dg_{i}$ and $\Dg_{j}$ may be two discrete measures with different masses. So, the transportation plan $\gamma$ is bijective between $\Dg_{i} \cup \Delta$ and $\Dg_{j} \cup \Delta$ instead of between $\Dg_{i}$ and $\Dg_{j}$. \cite{carriere17asliced}, for instance, used Wasserstein distance between $\Dg_{i}$ and $\Dg_{j}$ where its transportation plans operate between $\Dg_{i} \cup \Dg_{j\Delta}$ and $\Dg_{j} \cup \Dg_{i\Delta}$ (nonnegative, not necessarily normalized measures with same masses). Here, we denote $\Dg_{i\Delta} := \left\{\Pi_{\Delta}(u) \mid u \in \Dg_{i} \right\}$ where $\Pi_{\Delta}(u)$ is a projection of a point $u$ on the diagonal set $\Delta$. Following this line of work, we also consider a distance between two measures $\Dg_{i} \cup \Dg_{j\Delta}$ and $\Dg_{i} \cup \Dg_{j\Delta}$ as a distance between $\Dg_{i}$ and $\Dg_{j}$ for the Fisher information metric. 

\begin{definition}\label{def:RM}
Let $\Dg_{i}, \Dg_{j}$ be two finite and bounded persistence diagrams. The Fisher information metric between $\Dg_{i}$ and $\Dg_{j}$ is defined as follows,
\begin{equation}\label{equ:RMPD}
d_\FIM(\Dg_{i}, \Dg_{j}) := d_{\Pp} \! \left(\rho_{\left(\Dg_{i} \cup \Dg_{j\Delta}\right)}, \rho_{\left(\Dg_{j} \cup \Dg_{j\Delta}\right)} \right).
\end{equation}
\end{definition}

\begin{lemma}\label{lem:RM}
Let $\DD$ be the set of bounded and finite persistent diagrams. Then, $\left(d_{\FIM} - \tau \right)$ is negative definite on $\DD$ for all $\tau \ge \frac{\pi}{2}$.
\end{lemma}

\begin{proof}
Let consider the function $\tau - \arccos(\xi)$ where $\tau \ge \frac{\pi}{2}$ and $\xi \in \left[0, 1\right]$, then apply the Taylor series expansion for $\arccos(\xi)$ at $0$, we have 
\[
\tau - \arccos(\xi) = \tau - \frac{\pi}{2} + \sum_{i=0}^{\infty} \frac{(2i)!}{2^{2i} (i!)^2 (2i + 1)} x^{2i + 1}.
\] 
So, all coefficients of the Taylor series expansion are nonnegative. Following \citep{schoenberg1942positive} (Theorem 2, p. 102), for $\tau \ge \frac{\pi}{2}$ and $\xi \in \left[0, 1\right]$, $\tau - \arccos(\xi)$ is positive definite. Consequently, $\arccos(\xi) - \tau$ is negative definite. Furthermore, for any PDs $\Dg_i$ and $\Dg_j$ in $\DD$, we have 
\[
0 \le \int \sqrt{\bar{\rho}_i(x) \bar{\rho}_j(x)} \mathrm{d}x \le 1,
\]
where we denote $\bar{\rho}_i = \rho_{\left(\Dg_{i} \cup \Dg_{j\Delta}\right)}$ and $\bar{\rho}_j = \rho_{\left(\Dg_{j} \cup \Dg_{i\Delta}\right)}$. The lower bound is due to nonnegativity of the probability simplex while the upper bound follows from the Cauchy-Schwarz inequality. Hence, $d_{\FIM} - \tau$ is negative definite on $\DD$ for all $\tau \ge \frac{\pi}{2}$.
\end{proof}

Based on Lemma \ref{lem:RM}, we propose a positive definite kernel for PDs under the Fisher information geometry by following \citep{berg1984harmonic} (Theorem 3.2.2, p.74), namely the Persistence Fisher kernel,
\begin{equation}\label{equ:RMKernel}
k_{\PF}(\Dg_i, \Dg_j) := \exp\left(- td_{\FIM}(\Dg_i, \Dg_j)\right),
\end{equation} 
where $t$ is a positive scalar since we can rewrite the Persistence Fisher kernel as $k_{\PF}(\Dg_i, \Dg_j) = \alpha \exp\left(- t\left(d_{\FIM}(\Dg_i, \Dg_j) - \tau\right) \right)$ where $\tau \ge \frac{\pi}{2}$ and $\alpha = \exp\left(-t \tau \right) > 0$.

To the best of our knowledge, the $k_{\PF}$ is the first kernel relying on the Fisher information geometry for measuring the similarity of PDs. Moreover, the $k_{\PF}$ is positive definite \textit{without any approximation}.

\begin{remark} 
Let $\SS_{+} := \left\{\nu \mid \int \nu^2(x) \mathrm{d}x = 1, \nu(x) \ge 0 \right\}$ be the positive orthant of the sphere, and define the Hellinger mapping $h(\cdot) := \sqrt{\cdot}$, where the square root is an element-wise function which transforms the probability simplex $\PP$ into $\SS_{+}$. The Fisher information metric between $\rho_i$ and $\rho_j$ in $\PP$ (Equation (\ref{equ:dFIM})) is equivalent to the geodesic distance between $h(\rho_i)$ and $h(\rho_j)$ in $\SS_{+}$. From \citep{levy1965processus}, the geodesic distance in $\SS_{+}$ is a measure definite kernel distance. Following \citep{istas2012manifold} (Proposition 2.8), the geodesic distance in $\SS_{+}$ is negative definite. This result is also noted in \citep{feragen2015geodesic}. From \citep{berg1984harmonic} (Theorem 3.2.2, p.74), the Persistence Fisher kernel is positive definite. Therefore, our proof technique is not only independent and direct for the Fisher information metric on the probability simplex without relying on the geodesic distance on $\SS_{+}$, but also valid for the case of infinite dimensions due to \citep{schoenberg1942positive} (Theorem 2, p. 102).
\end{remark}

\begin{remark} 
A closely related kernel to the Persistence Fisher kernel is the diffusion kernel \citep{lafferty2005diffusion} (p. 140), based on the heat equation on the Riemannian manifold defined by the Fisher information metric to exploit the geometric structure of statistical manifolds. A generalized family of kernels for the diffusion kernel is exploited in \citep{jayasumana2015kernel, feragen2015geodesic}. To the best of our knowledge, the diffusion kernel has not been used for measuring the similarity of PDs. If one uses the Fisher information metric (Definition \ref{def:RM}) for PDs, and then plug the distance into the diffusion kernel, one obtains a similar form to our proposed Persistence Fisher kernel. A slight difference is that the diffusion kernel relies on $d_{\FIM}^2$ while the Persistence Fisher kernel is built upon $d_{\FIM}$ itself. However, the Persistence Fisher kernel is positive definite while it is unclear whether the diffusion kernel is positive definite\footnote{Although the heat kernel is positive definite, the diffusion kernel on the probability simplex---the heat kernel on multinomial manifold---does not have an explicit form. In practice, the diffusion kernel equation \citep{lafferty2005diffusion} (p. 140) is only its first-order approximation.}.
\end{remark}

\paragraph{Computation.} Given two finite PDs $\Dg_i$ and $\Dg_j$ with cardinalities bounded by $N$, in practice, we consider a finite set $\Theta := \Dg_i \cup \Dg_{j\Delta} \cup \Dg_j \cup \Dg_{i\Delta}$ without multiplicity in $\RR^2$ for smoothed and normalized measures $\rho_{(\cdot)}$ (Equation \ref{equ:NSprob})\footnote{We leave the computation with an \textit{infinite} set $\Theta$ for future work.}. Then, let $m$ be the cardinality of $\Theta$, we have $m \le 4N$. Consequently, the time complexity of $\rho_{(\cdot)}$ is $O(Nm)$. For acceleration, we propose to apply the Fast Gauss Transform \citep{greengard1991fast, morariu2009automatic} to approximate the sum of Gaussian functions in $\rho_{(\cdot)}$ with a bounded error. The time complexity of $\rho_{(\cdot)}$ is reduced from $O(Nm)$ to $O(N + m)$. Due to the low dimension of points in PDs ($\RR^2$), this approximation by the Fast Gauss Transform is very efficient in practice. Additionally, $d_{\Pp}$ (Equation (\ref{equ:dFIM})) is evaluated between two points in the $m$-dimensional probability simplex $\PP_{m-1}$ where $\PP_{m-1} := \left\{x \mid x \in \RR_{+}^m, \norm{x}_1 = 1\right\}$. So, the time complexity of the Persistence Fisher kernel $k_{\PF}$ between two smoothed and normalized measures is $O(m)$. Hence, the time complexity of $k_{\PF}$ between $\Dg_i$ and $\Dg_j$ is $O(N^2)$, or $O(N)$ for the acceleration version with Fast Gauss Transform. We summarize the computation of $d_{\FIM}$ in Algorithm $\ref{alg:dFIM}$, where the second and third steps can be approximated with a bounded error via Fast Gaussian Transform with a linear time complexity $O(N)$. Source code for Algorithm $\ref{alg:dFIM}$ can be obtained in \url{http://github.com/lttam/PersistenceFisher}. We recall that the time complexity of the Wasserstein distance between $\Dg_i$ and $\Dg_j$ is $O(N^3 \log N)$ \citep{pele2009fast} (\S 2.1). For the Sliced Wasserstein distance (an approximation of Wasserstein distance), the time complexity is $O(N^2 \log N)$ \citep{carriere17asliced}, or $O(M N \log N)$ for its approximation with $M$ projections \citep{carriere17asliced}. We also summary a comparison for the time complexity and metric preservation of $k_{\PF}$ and related kernels for PDs in Table \ref{tb:TimeComplexity}.

\begin{algorithm} % enter the algorithm environment
\caption{Compute $d_{\FIM}$ for persistence diagrams} % give the algorithm a caption
\label{alg:dFIM} % and a label for \ref{} commands later in the document
\begin{algorithmic}[1] % enter the algorithmic environment
    \REQUIRE Persistence diagrams $\Dg_i$, $\Dg_j$, and a bandwith $\sigma>0$ for smoothing
    \ENSURE $d_{\FIM}$
    \STATE Let $\Theta \leftarrow \Dg_i \cup \Dg_{j\Delta} \cup \Dg_j \cup \Dg_{i\Delta}$ (a set for smoothed and normalized measures)
     \STATE Compute $\bar{\rho_i} = \rho_{\left(\Dg_{i} \cup \Dg_{j\Delta}\right)} \leftarrow \left[ \frac{1}{Z} \sum_{u \in \Dg_{i} \cup \Dg_{j\Delta}} \mathtt{N}(x; u, \sigma I)\right]_{x \in \Theta}$ \\ \qquad \qquad where $Z \leftarrow \sum_{x \in \Theta} \sum_{u \in  \Dg_{i} \cup \Dg_{j\Delta}} \mathtt{N}(x; u, \sigma I) $
     \STATE Compute $\bar{\rho_j} = \rho_{\left(\Dg_{j} \cup \Dg_{i\Delta}\right)}$ similarly as $\bar{\rho_i}$.
     \STATE Compute $d_{\FIM} \leftarrow \arccos\left(\left<\sqrt{\bar{\rho_i}}, \sqrt{\bar{\rho_j}} \right>\right)$ where $\left< \cdot, \cdot \right>$ is a dot product and $\sqrt{\cdot}$ is element-wise.	
\end{algorithmic}
\end{algorithm}

\begin{table*}[t]
\centering
\caption{A comparison for time complexities and metric preservation of kernels for PDs. Noted that the SW kernel is built upon the SW distance (an \textit{approximation} of Wasserstein metric) while the PF kernel uses the Fisher information metric \textit{without approximation}.}
\label{tb:TimeComplexity}
\begin{tabular}{|l|c|c|c|c|}
\hline
                                                                                & $k_{\PSS}$                        & $k_{\PWG}$                        & $k_{\SW}$                         & $k_{\PF}$                         \\ \hline
Time complexity                                                                 & $O(N^2)$                          & $O(N^2)$                           & $O(N^2\log N)$                    & $O(N^2)$                          \\ \hline
\begin{tabular}[c]{@{}c@{}}Time complexity with approximation\end{tabular} & $O(N)$                            & $O(N)$                            & $O(MN\log N)$                     & $O(N)$                            \\ \hline
Metric preservation                                                             &                                   &                                   & $\checkmark$                      & $\checkmark$                      \\ \hline
\end{tabular}
 \vspace{-10pt}
\end{table*}

\section{Theoretical Analysis}\label{sec:Theoretical}

In this section, we analyze for the Persistence Fisher kernel $k_{\PF}$ (in Equation (\ref{equ:RMKernel})) where the Hellinger mapping $h$ of a smoothed and normalized measure $\rho_{(\cdot)}$ is on the positive orthant of the $d$-dimension unit sphere $\SS^{+}_{d-1}$ where $\SS^{+}_{d-1} := \left\{x \mid x \in \RR_{+}^d, \norm{x}_2 = 1\right\}$\footnote{It is corresponding to a finite set $\Theta$.}. Let $\Dg_{i}, \Dg_{j}$ be PDs in the set $\DD$ of bounded and finite PDs, and $\mu$ be the uniform probability distribution on $\SS^{+}_{d-1}$. We denote $x_i$ and $x_j \in \SS^{+}_{d-1}$ as corresponding mapped points through the Hellinger mapping $h$ of smoothed and normalized measures $\rho_{(\Dg_{i} \cup \Dg_{j\Delta})}$ and $\rho_{(\Dg_{j} \cup \Dg_{i\Delta})}$ respectively. Then, we rewrite the Persistence Fisher kernel between $x_i$ and $x_j$ as follows,
\begin{equation}\label{equ:RMKsphere}
k_{\PF}(x_i, x_j) = \exp{\left(-t \arccos\left(\left<x_i, x_j\right> \right)\right)}.
\end{equation} 

\paragraph{Eigensystem.} Let $T_{k_{\PF}}: L_2(\SS^{+}_{d-1}, \mu) \to L_2(\SS^{+}_{d-1}, \mu)$ be the integral operator induced by the Persistence Fisher kernel $k_{\PF}$, which is defined as 
\[
\left(T_{k_{\PF}} f \right) (\cdot) := \int k_{\PF}(x, \cdot) f(x) \mathrm{d}\mu(x).
\]
Following \citep{smola2001regularization} (Lemma 4), we derive an eigensystem of the integral operator $T_{k_{\PF}}$ as in Proposition \ref{pro:Eigensystem}.
\begin{proposition}\label{pro:Eigensystem}
Let $\left\{a_i\right\}_{i\ge0}$ be the coefficients of Legendre polynomial expansion of the Persistence Fisher kernel $k_{\PF}(x, z)$ defined on $\SS^{+}_{d-1} \times \SS^{+}_{d-1}$ as in Equation (\ref{equ:RMKsphere}), 
\begin{equation}\label{equ:LegPolyExpansion}
k_{\PF}(x, z) = \sum_{i=0}^{\infty} a_i P_{i}^{d}(\left<x, z\right>),
\end{equation}
where $P_{i}^d$ is the associated Legendre polynomial of degree $i$. Let $\left| \SS_{d-1} \right| := \frac{2 \pi^{d/2}}{\Gamma(d/2)}$ denote the surface of $\SS_{d-1}$ where $\Gamma(\cdot)$ is the Gamma function, $N(d, i) := \frac{(d+ 2i - 2)(d + i - 3)!}{(d-2)!i!}$ denote the multiplicity of spherical harmonics of order $i$ on $\SS_{d-1}$, and $\left\{Y_{i, j}^{d}\right\}_{1 \le j \le N(d, i)}$ denote any fixed orthonormal basis for the subspace of all homogeneous harmonics of order $i$ on $\SS_{d-1}$. Then, the eigensystem $\left(\lambda_{i, j}, \phi_{i, j}\right)$ of the integral operator $T_{k_{\PF}}$ induced by the Persistence Fisher kernel $k_{\PF}$ is 
\begin{eqnarray}
&& \phi_{i, j} = Y_{i, j}^{d},  \\
&& \lambda_{i, j} = \frac{a_i \left| \SS_{d-1} \right|}{N(d, i)} \label{equ:Eigenvalue}
\end{eqnarray}
of multiplicity $N(d, i)$.
\end{proposition}
\begin{proof}
From the Addition Theorem \citep{muller2012analysis} (Theorem 2, p. 18) and the Funk-Hecke formula \citep{muller2012analysis} (\S4, p. 29), we have $\sum_{j=1}^{N(d, i)} Y_{i, j}^d(x) Y_{i, j}^d(z) = \frac{N(d, i)}{\left| \SS_{d-1} \right|} P_i^d\left( \left<x, z \right>\right)$, then replace $P_i^d$ into Equation (\ref{equ:LegPolyExpansion}), and note that $\int_{\SS_{d-1}} Y_{i, j}^{d}(x) Y_{i', j'}^{d}(x) \mathrm{d}x = \delta_{i, i'} \delta_{j, j'}$, we complete the proof.
\end{proof}

\begin{proposition}
All coefficients of Legendre polynomial expansion of the Persistence Fisher kernel are nonnegative.
\end{proposition}
\begin{proof}
From Lemma \ref{lem:RM}, the $k_{\PF}$ is positive definite. Applying \cite{schoenberg1942positive} (Theorem 1, p. 101) for $k_{\PF}$ defined on $\SS^{+}_{d-1} \times \SS^{+}_{d-1}$ as in Equation (\ref{equ:RMKsphere}), we obtain the result.
\end{proof}

The eigensystem of the integral operator $T_{k_{\PF}}$ induced by the PF kernel plays an important role to derive generalization error bounds for kernel machines with the proposed PF kernel via covering numbers and Rademacher averages as in Proposition \ref{pro:CoveringNumber} and Proposition \ref{pro:RademacherAverages} respectively.

% \subsection{Covering Numbers}
\paragraph{Covering numbers.} Given a set of finite points $\mathtt{S} = \left\{x_i \mid x_i \in \SS_{d-1}^{+}, d \ge 3 \right\}$, the Persistence Fisher kernel hypothesis class with $R$-bounded weight vectors for $\mathtt{S}$ is defined as follows
\[
\Ff_{R}(\mathtt{S}) = \left\{\mathtt{f} \mid \mathtt{f}(x_i) = \left<w, \phi\left( x_i\right) \right>_{\Hh}, \norm{w}_{\Hh} \le R \right\},
\]
where $\left<\phi\left( x_i\right), \phi\left( x_j\right) \right>_{\Hh} = k_{\PF}(x_i, x_j)$. $\left<\cdot, \cdot\right>_{\Hh}$ and $\norm{\cdot}_{\Hh}$ are an inner product and a norm in the corresponding Hilbert space respectively. Following \citep{guo1999covering}, we derive bounds on the generalization performance of the PF kernel on kernel machines via the covering numbers $\Nn(\cdot, \Ff_{R}(\mathtt{S}))$ \citep{shalev2014understanding} (Definition 27.1, p. 337) as in Proposition \ref{pro:CoveringNumber}.
\begin{proposition}\label{pro:CoveringNumber}
Assume the number of non-zero coefficients $\left\{ a_i \right\}$ in Equation (\ref{equ:LegPolyExpansion}) is finite, and $r$ is the maximum index of the non-zero coefficients. Let $q := \argmax_{i} \lambda_{i, \cdot}$, choose $\alpha \in \NN$ such that $\alpha <  \left( \frac{\lambda_{q, \cdot}}{\lambda_{i, \cdot}} \right)^{\!\! \frac{N(d, q)}{2}}$ with $i \ne q$, and define $\varepsilon := 6 R \sqrt{N(d, r) \left( a_q \alpha^{-2/N(d, q)} + \sum_{i=0, i \ne q}^{\infty} a_{i} \right)}$. Then, 
\[
\sup_{x_i \in \mathtt{S}} \Nn(\varepsilon, \Ff_{R}(\mathtt{S})) \le \alpha.
\]
\end{proposition}

\begin{proof}
From \citep{minh2006mercer} (Lemma 3), we have $\norm{Y_{i, j}^{d}}_{\infty} \le \sqrt{\frac{N(d, i)}{\left| \SS_{d-1}\right|}}$. It is easy to check that $\forall d\ge 3, i \!\ge\! j \!\ge\! 0$, we have $N(d, i) \ge N(d, j)$. Therefore, following Proposition \ref{pro:Eigensystem}, all eigenfunctions of $k_{\PF}$ satisfy that $\norm{Y_{i, j}^{d}}_{\infty} \le \sqrt{\frac{N(d, r)}{\left| \SS_{d-1} \right|}}$. Additionally, the multiplicity of $\lambda_{i, \cdot}$ is $N(d, i)$, and $N(d, i) \lambda_{i, \cdot} = a_i \left| \SS_{d-1}\right|$ (Equation (\ref{equ:Eigenvalue})). Hence, from \citep{guo1999covering} (Theorem 1), we obtain the result.
\end{proof}

\paragraph{Rademacher averages.} We provide a different family of generalization error bounds via Rademacher averages \citep{bartlett2005local}. By plugging the eigensystem of the PF kernel as in Proposition \ref{pro:Eigensystem} into the localized averages of function classes based on the PF kernel with respect to the uniform probability distribution $\mu$ on $\SS_{d-1}^{+}$ \citep{mendelson2003performance} (Theorem 2.1), we obtain a bound as in Proposition \ref{pro:RademacherAverages}.

\begin{proposition}\label{pro:RademacherAverages}
Let $\left\{ x_i \right\}_{1\le i\le m}$ be independent, distributed according to the uniform probability distribution $\mu$ on $\SS_{d-1}^{+}$, denote $\left\{ \sigma_i \right\}_{1 \le i \le m}$ for independent Rademacher random variables, $\Hh_{k_{\PF}}$ for the unit ball of the reproducing kernel Hilbert space corresponding with the Riemanian manifold kernel $k_{\PF}$, and let $q := \argmax_{i} \lambda_{i, \cdot}$. If $\lambda_{q, \cdot} \ge 1/m$, for $\tau \ge 1/(m \left|\SS_{d-1}\right|)$, let $\Psi(\tau) := \sqrt{\left| \SS_{d-1} \right| \left( \sum\limits_{a_i < \tau N(d, i)} \!\!\!\!\!\!\! a_i +  \tau \!\!\!\!\!\!\!\! \sum\limits_{a_i \ge \tau N(d, i)} \!\!\!\!\!\!\! N(d, i) \right)}$, then there are absolute constants $C_{\ell}$ and $C_{u}$ which satisfy 
\begin{equation}
C_{\ell} \Psi(\tau) \le \EE \sup_{\substack{\mathtt{f} \in \Hh_{k_{\PF}} \\ \frac{\EE_{\mu}\mathtt{f}^2}{ \left|\SS_{d-1}\right|} \le \tau}} \left| \sum_{i=1}^{m} \sigma_i \mathtt{f}(x_i) \right| \le C_{u} \Psi(\tau),
\end{equation}
where $\EE$ is an expectation.
\end{proposition} 

From Proposition \ref{pro:CoveringNumber} and Proposition \ref{pro:RademacherAverages}, a decay rate of the eigenvalues of the integral operator $T_{k_{\PF}}$ is relative with the capacity of the kernel learning machines. When the decay rate of the eigenvalues is large, the capacity of kernel machines is reduced. So, if the training error of kernel machines is small, then it can lead to better bounds on generalization error. The resulting bounds for both the covering number (Proposition \ref{pro:CoveringNumber}) and the Rademacher averages (Proposition \ref{pro:RademacherAverages}) are essentially the same as the standard ones for a Gaussian kernel on a Euclidean space.

\paragraph{Bounding for $k_{\PF}$ induced squared distance with respect to $d_{\FIM}$.} The squared distance induced by the PF kernel, denoted as $d^2_{k_{\PF}}$,  can be computed by the Hilbert norm of the difference between two corresponding mappings. Given two persistent diagram $\Dg_i$ and $\Dg_j$, we have 
\[
d^2_{k_{\PF}}\left( \Dg_i, \Dg_j \right) := k_{\PF}\left( \Dg_i, \Dg_i \right) + k_{\PF} \left( \Dg_j, \Dg_j \right) - 2 k_{\PF}\left( \Dg_i, \Dg_j \right).
\]
We recall that $k_{\PF}$ is based on the Fisher information geometry. So, it is of interest to bound the PF kernel induced squared distance $d^2_{k_{\PF}}$ with respect to the corresponding Fisher information metric $d_{\FIM}$ between PDs as in Lemma \ref{lm:BoundKIDist}.
\begin{lemma}\label{lm:BoundKIDist}
Let $\DD$ be the set of bounded and finite persistent diagrams. Then, $\forall \Dg_i, \Dg_j \in X$, 
\[
d^2_{k_{\PF}}(\Dg_i, \Dg_j) \le 2t d_{\FIM}(\Dg_i, \Dg_j),
\]
where $t$ is a parameter of $k_{\PF}$.
\end{lemma}

\begin{proof}
We have $d^2_{k_{\PF}}(\Dg_i, \Dg_j) = 2\left(1 - k_{\PF}\left( \Dg_i, \Dg_j \right)\right) = 2(1 - \exp\left(-t d_{\FIM}\left( \Dg_i, \Dg_j \right) \right) \le 2td_{\FIM}\left( \Dg_i, \Dg_j \right)$, since $1 - \exp(-a) \le a, \forall a \ge 0$.
\end{proof}

From Lemma \ref{lm:BoundKIDist}, it implies that the Persistence Fisher kernel is stable on Riemannian geometry in a similar sense as the work of \cite{kwitt2015statistical}, and \cite{reininghaus2015stable} on Wasserstein geometry.

\paragraph{Infinite divisibility for the Persistence Fisher kernel.}
\begin{lemma}
The Persistence Fisher kernel $k_{\PF}$ is infinitely divisible.
\end{lemma}
\begin{proof}
For $\mathtt{m} \in \NN^{*}$, let $k_{\PF_{\mathtt{m}}} := \exp\left({- \frac{t}{\mathtt{m}} d_{\FIM}}\right)$, so $\left(k_{\PF_{\mathtt{m}}} \right)^{\mathtt{m}} = k_{\PF}$ and note that $k_{\PF_{\mathtt{m}}}$ is positive definite. Hence, following \cite{berg1984harmonic} (\S3, Definition 2.6, p. 76), we have the result. 
\end{proof}

As for infinitely divisible kernels, the Gram matrix of the PF kernel does not need to be recomputed for each choice of $t$ (Equation (\ref{equ:RMKernel})), since it suffices to compute the Fisher information metric between PDs in training set only once. This property is shared with the Sliced Wasserstein kernel \citep{carriere17asliced}. However, neither Persistence Scale Space kernel \citep{reininghaus2015stable} nor Persistence Weighted Gaussian kernel \citep{kusano2016persistence} has this property.

\section{Experimental Results}\label{sec:Experimental}

\begin{table}
%\vspace{-20pt}
\centering
\caption{Results on SVM classification. The averaged accuracy (\%) and standard deviation are shown.}
\label{tb:OrbitObjectshape}
\begin{tabular}{|l|l|l|}
\hline
\multicolumn{1}{|c|}{}           & \multicolumn{1}{c|}{MPEG7}            & \multicolumn{1}{c|}{Orbit}            \\ \hline
{$k_{\PSS}$} & \multicolumn{1}{c|}{$73.33 \pm 4.17$} & \multicolumn{1}{c|}{$72.38 \pm 2.41$} \\ \hline
$k_{\PWG}$                       & $74.83 \pm 4.36$                      & $76.63 \pm 0.66$                      \\ \hline
$k_{\SW}$                        & $76.83 \pm 3.75$                      & $83.60 \pm 0.87$                      \\ \hline
Prob+$k_G$                       & $55.83 \pm 5.45$                      & $72.89 \pm 0.62$                      \\ \hline
Tang+$k_G$                       & $66.17 \pm 4.01$                      & $77.32 \pm 0.72$                      \\ \hline
\textbf{\boldmath{$k_{\PF}$}}                        & \textbf{80.00 $\pm$ 4.08}                      & \textbf{85.87 $\pm$ 0.77}                      \\ \hline
\end{tabular}
\vspace{-10pt}
\end{table}

We evaluated the Persistence Fisher kernel with support vector machines (SVM) on many benchmark datasets. We consider five baselines as follows: (i) the Persistence Scale Space kernel ($k_{\PSS}$), (ii) the Persistence Weighted Gaussian kernel ($k_{\PWG}$), (iii) the Sliced Wasserstein kernel ($k_{\SW}$), (iv) the smoothed and normalized measures in the probability simplex with the Gaussian kernel (Prob + $k_G$), and (v) the tangent vector representation \citep{anirudh2016riemannian} with the Gaussian kernel (Tang + $k_G$). Practically, Euclidean metric is not a suitable geometry for the probability simplex \citep{le2015adaptive, le2015unsupervised}. So, the (Prob + $k_G$) approach may not work well for PDs. For hyper-parameters, we typically choose them through cross validation. For baseline kernels, we follow their corresponding authors to form sets of hyper-parameter candidates, and the bandwidth of the Gaussian kernel in (Prob + $k_G$) and (Tang + $k_G$) is chosen from $10^{\left\{-3:1:3 \right\}}$. For the Persistence Fisher kernel, there are $2$ hyper-parameters: $t$ (Equation (\ref{equ:RMKernel})) and $\sigma$ for smoothing measures (Equation (\ref{equ:NSprob})). We choose $1/t$ from $\left\{\mathtt{q}_{1}, \mathtt{q}_{2}, \mathtt{q}_{5}, \mathtt{q}_{10}, \mathtt{q}_{20}, \mathtt{q}_{50}  \right\}$ where $\mathtt{q}_s$ is the $s\%$ quantile of a subset of Fisher information metric between PDs, observed on the training set, and $\sigma$ from $\left\{ 10^{-3:1:3}\right\}$. For SVM, we use Libsvm (one-vs-one) \citep{chang2011libsvm} for multi-class classification, and choose a regularization parameter of SVM from $\left\{ 10^{-2:1:2} \right\}$. For PDs, we used the DIPHA toolbox\footnote{https://github.com/DIPHA/dipha}.

\begin{wraptable}{r}{0.6\textwidth}
\vspace{-19pt}
\centering
\caption{Computational time (seconds) with approximation. For each dataset, the first number in the parenthesis is the number of PDs while the second one is the maximum number of points in PDs.}
\label{tb:ComputationalTime}
\begin{tabular}{|l|c|c|c|c|}
\hline
\multicolumn{1}{|c|}{}          & \begin{tabular}[c]{@{}c@{}}Orbit\\ (5K/300)\end{tabular} & \begin{tabular}[c]{@{}c@{}}MPEG7\\ (200/80)\end{tabular} & \begin{tabular}[c]{@{}c@{}}Granular\\ (35/20.4K)\end{tabular} & \begin{tabular}[c]{@{}c@{}}SiO$_2$\\ (80/30K)\end{tabular} \\ \hline
{$k_{\SW}$} & 6473                                                     & 1.55                                                     & 8.30                                                          & 249                                                     \\ \hline
$k_{\PWG}$                      & 8756                                                     & 5.23                                                     & 17.44                                                         & 288                                                     \\ \hline
$k_{\PSS}$                      & 11024                                                    & 7.51                                                     & 38.14                                                         & 515                                                     \\ \hline
\textbf{\boldmath{$k_{\PF}$}}                        & \textbf{9891}                                                     & \textbf{6.63}                                                     & \textbf{22.70}                                                         & \textbf{318}                                                     \\ \hline
\end{tabular}
\vspace{-10pt}
\end{wraptable}

%\begin{wraptable}{r}{0.45\textwidth}
%\vspace{-22pt}
%\centering
%\caption{Results on SVM classification. The averaged accuracy (\%) and standard deviation are shown.}
%\label{tb:OrbitObjectshape}
%\begin{tabular}{|l|l|l|}
%\hline
%\multicolumn{1}{|c|}{}           & \multicolumn{1}{c|}{MPEG7}            & \multicolumn{1}{c|}{Orbit}            \\ \hline
%{$k_{\PSS}$} & \multicolumn{1}{c|}{$73.33 \pm 4.17$} & \multicolumn{1}{c|}{$72.38 \pm 2.41$} \\ \hline
%$k_{\PWG}$                       & $74.83 \pm 4.36$                      & $76.63 \pm 0.66$                      \\ \hline
%$k_{\SW}$                        & $76.83 \pm 3.75$                      & $83.60 \pm 0.87$                      \\ \hline
%Prob+$k_G$                       & $55.83 \pm 5.45$                      & $72.89 \pm 0.62$                      \\ \hline
%Tang+$k_G$                       & $66.17 \pm 4.01$                      & $77.32 \pm 0.72$                      \\ \hline
%\textbf{\boldmath{$k_{\PF}$}}                        & \textbf{80.00 $\pm$ 4.08}                      & \textbf{85.87 $\pm$ 0.77}                      \\ \hline
%\end{tabular}
%\vspace{-15pt}
%\end{wraptable}

%%%%%%%%%%%%%%%%%%%%
\subsection{Orbit Recognition}
It is a synthesized dataset proposed by \citep{adams2017persistence} (\S 6.4.1) for \textit{linked twist map} which is a discrete dynamical system modeling flow. The linked twist map is used to model flows in DNA microarrays \citep{hertzsch2007dna}. Given a parameter $r > 0$, and initial positions $(s_0, t_0) \in \left[0, 1\right]^2$, its orbit is described as $s_{i+1} = s_i + r t_i (1 - t_i) \mod 1$, and $t_{i+1} = t_i + r s_{i+1} (1 - s_{i+1}) \mod 1$. \cite{adams2017persistence} proposed $5$ classes, corresponding to $5$ different parameters $r = 2.5, 3.5, 4, 4.1, 4.3$. For each parameter $r$, we generated $1000$ orbits where each orbit has $1000$ points with random initial positions. We randomly split $70\%/30\%$ for training and test, and repeated $100$ times. We extract only $1$-dimensional topological features with Vietoris-Rips complex filtration \citep{edelsbrunner2008persistent} for PDs. The accuracy results on SVM are summarized in the third column of Table \ref{tb:OrbitObjectshape}. The PF kernel outperforms all other baselines. The (Prob + $k_{G}$) does not performance well as expected. Moreover, the $k_{\PF}$ and $k_{\SW}$ which enjoy the Fisher information geometry and Wasserstein geometry for PDs respectively, clearly outperform other approaches. As in the second column of Table \ref{tb:ComputationalTime}, the computational time of $k_{\PF}$ is faster than $k_{\PSS}$, but slower than $k_{\SW}$ and $k_{\PWG}$ for PDs.

\subsection{Object Shape Classification}
We consider a 10-class subset\footnote{The 10-classes are: apple, bell, bottle, car, classic, cup, device0, face, Heart and key.} of MPEG7 object shape dataset \citep{latecki2000shape}. Each class has 20 samples. We resize each image such that its length is shorter or equal $256$, and extract a boundary for object shapes before computing PDs. For simplicity, we only consider $1$-dimensional topological features with the traditional Vietoris-Rips complex filtration \citep{edelsbrunner2008persistent} for PDs\footnote{A more advanced filtration for this task was proposed in \citep{turner2014persistent}.}. We also randomly split $70\%/30\%$ for training and test, and repeated $100$ times. The accuracy results on SVM are summarized in the second column of Table \ref{tb:OrbitObjectshape}. The Persistence Fisher kernel compares favorably with other baseline kernels for PDs. All approaches based on the implicit representation via kernels for PDs outperform ones based on the explicit vector representation with Gaussian kernel by a large margin. Additionally, the $k_{\PF}$ and $k_{\SW}$ also compares favorably with other approaches. As in the third column of Table \ref{tb:ComputationalTime}, the computational time of $k_{\PF}$ is comparative with $k_{\PWG}$ and $k_{\PSS}$, but slower than the $k_{\SW}$.

\subsection{Change Point Detection for Material Data Analysis}

We evaluated the proposed kernel for the change point detection problem for material data analysis on granular packing system \citep{francois2013geometrical} and SiO$_2$\citep{nakamura2015persistent} datasets. We use the kernel Fisher discriminant ratio \citep{harchaoui2009kernel} (KFDR) as a statistical quantity and set $10^{-3}$ for the regularization of KFDR as in \citep{kusano2017kernelJMLR}. We use the ball model filtration to extract the $2$-dimensional topological features of PDs for granular packing system dataset, and $1$-dimensional topological features of PDs for SiO$_2$ dataset. We illustrate the KFDR graphs for the granular packing system and SiO$_2$ datasets in Figure \ref{fg:KFDR}. For granular tracking system dataset, all methods obtain the change point as the $23^{rd}$ index. They supports the observation result in \citep{anonymous72} (corresponding id = 23). For the SiO$_2$ datasets, all methods obtain the results within the supported range ($35 \le$ id $\le 50$) from the traditional physical approach \citep{elliott1983physics}. The $k_{\PF}$ compares favorably with other baseline approaches as in Figure \ref{fg:KFDR}. As in the fourth and fifth columns of Table \ref{tb:ComputationalTime}, $k_{\PF}$ is faster than $k_{\PSS}$, but slower than $k_{\SW}$ and $k_{\PWG}$. 

\section{Conclusions}\label{sec:Conclusions}
In this work, we propose the positive definite Persistence Fisher (PF) kernel for persistence diagrams (PDs). The PF kernel is relied on the Fisher information geometry \textit{without approximation} for PDs. Moreover, the proposed kernel has many nice properties from both theoretical and practical aspects such as stability, infinite divisibility, linear time complexity over the number of points in PDs, and improving performances of other baseline kernels for PDs as well as implicit vector representation with Gaussian kernel for PDs in many different tasks on various benchmark datasets.

\begin{figure}
%  \vspace{-25pt}
  \begin{center}
    \includegraphics[width=0.7\textwidth]{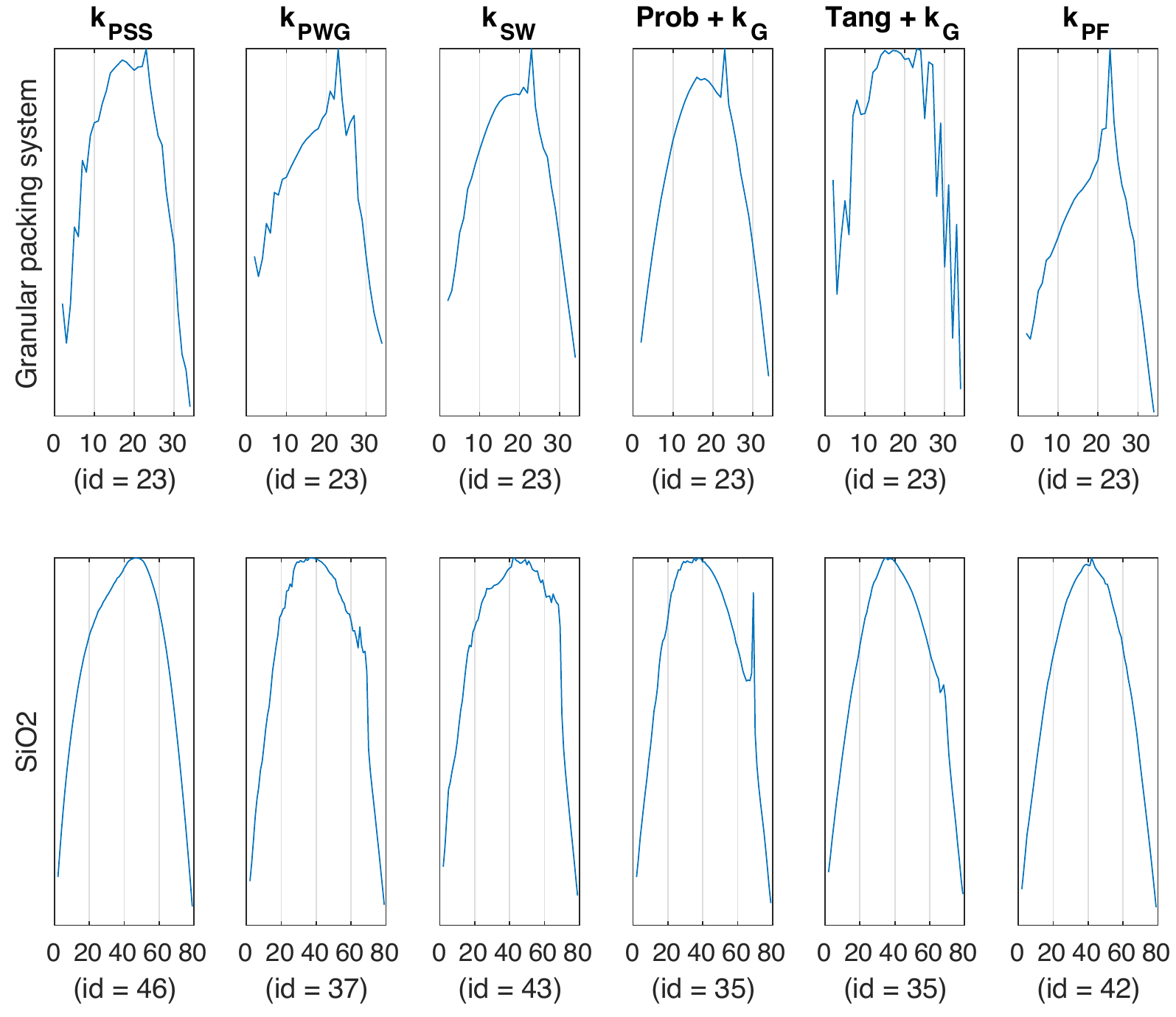}
  \end{center}
  \vspace{-10pt}
  \caption{The kernel Fisher discriminant ratio (KFDR) graphs.}
  \label{fg:KFDR}
  \vspace{-10pt}
\end{figure}

\subsubsection*{Acknowledgments}
We thank Ha Quang Minh, and anonymous reviewers for their comments. TL acknowledges the support of JSPS KAKENHI Grant number 17K12745. MY was supported by the JST PRESTO program JPMJPR165A.

\appendix

%%%%%%%%%%%%%%%%%%%%
\section{Some Traditional Filtrations for Persistence Diagrams}

\begin{figure}[ht]
\vskip 0.2in
\begin{center}
\centerline{\includegraphics[width=0.6\columnwidth]{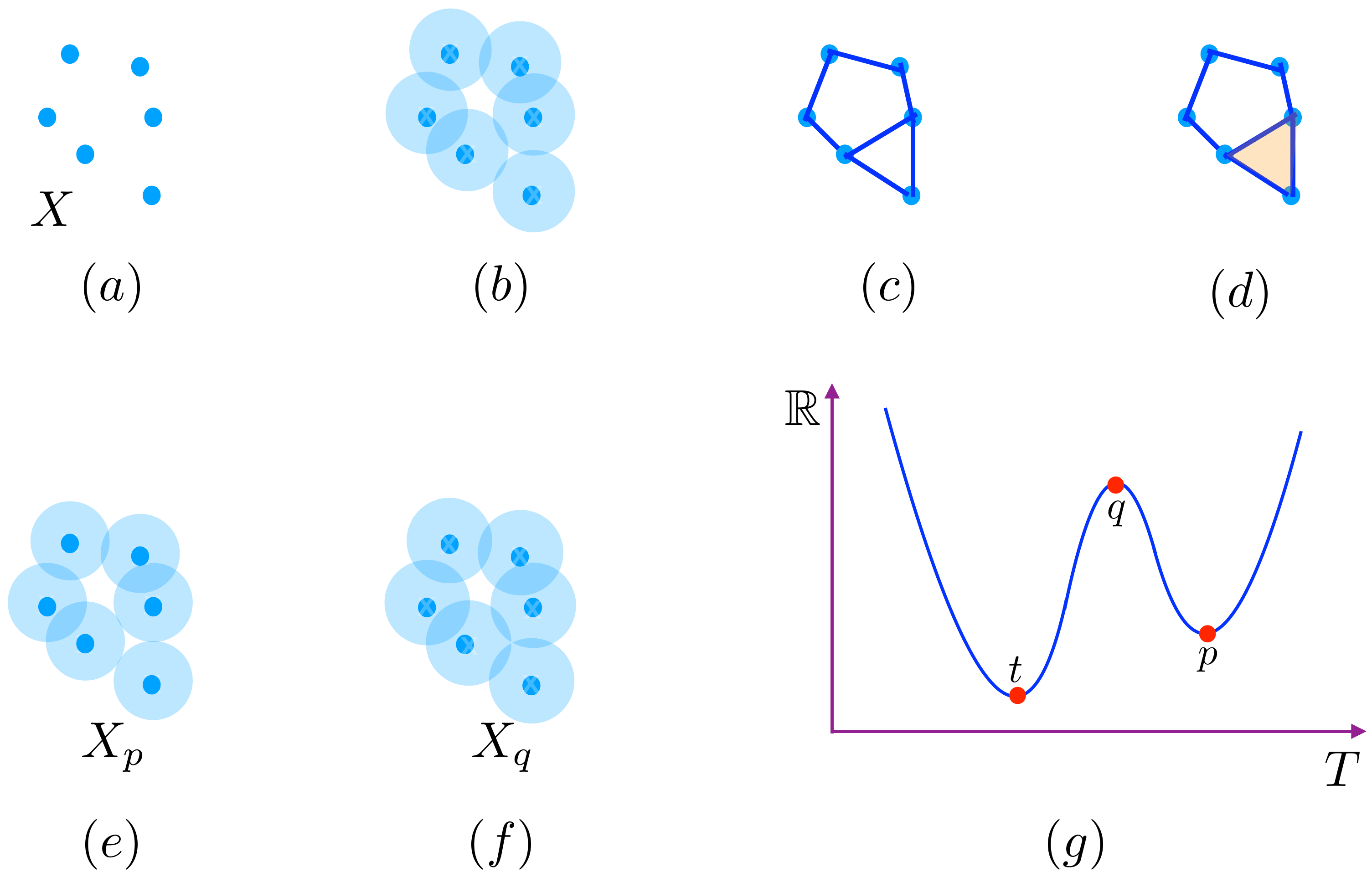}}
%\centerline{\includegraphics{PD_opt.pdf}}
\caption{An illustration for persistence diagrams with some popular filtrations. (a) A set of points as an input. (b) A ball model filtration. (c) Cech complex filtration. (d) Vietoris-Rips complex filtration (it has only $1$ ring since it contains a $2$-simplex, illustrated as the orange triangle). (e) An illustration of a birth of a ring in the ball model filtration. (f) An illustration of a death for a ring in the ball model filtration. (g) A sub-level set filtration (a connected component has a birth at $\mathfrak{F}_p$, and a death at $\mathfrak{F}_q$). In this illustration, both the ball model filtration and Cech complex filtration have $2$ rings, but there is only $1$ ring for Vietoris-Rips complex filtration due to the 2-simplex. For sub-level set filtration, there are 2 connected components $(p, q)$ and $(t, \infty)$. Hence, the persistence diagram of $0$-dimension topological feature is that $\Dg = \left\{ (p, q); (t, \infty)\right\}$.}
\label{fg:Persistence}
\end{center}
\vskip -0.2in
\end{figure}

We provide some traditional filtrations to illustrate persistence diagrams as follows,

\paragraph{Ball model filtration.} Let $X = \left\{ x_1, x_2, ..., x_m \right\}$ be a finite set in a metric space as in Figure \ref{fg:Persistence} (a), and $B(x, a)$ be a ball with a center $x$ and a radius $a$. We denote $X_{a} := \cup_{x_i \in X} B(x_i, a)$ for $a \ge 0$. For $a<0$, we define $X_{a} := \emptyset$. Therefore, $\left\{X_a \mid a \in \RR \right\}$ can be used as a filtration, illustrated in Figure \ref{fg:Persistence} (b). For example, Figure \ref{fg:Persistence} (e) shows a birth for a ring at $X_{p}$ while Figure \ref{fg:Persistence} (f) illustrates that the ring is death at $X_{q}$. Therefore, a point $(p, q)$ is in the persistence diagram of $1$-dimensional topological feature for the set $X$.

\paragraph{Cech complex filtration.} Given a set $X = \left\{ x_1, x_2, ..., x_m \right\}$ in a metric space $(T, d_T)$. For $a \ge 0$, we form a $t$-simplex from a $(t + 1)$-point subset $X_{t+1}$ of $X$ if there exist $x' \in M$, such that $d_T(x, x')\le a, \forall x \in X_{t+1}$. The set of all these simplices is called the Cech complex of $X$ with parameter $a \ge 0$, denoted as $\mathfrak{C}(X, a)$. For $a < 0$, we definite $\mathfrak{C}(X, a) := \emptyset$. Therefore, $\left\{ \mathfrak{C}(X, a) \mid a \in \RR \right\}$ can be considered as a filtration and illustrated in Figure \ref{fg:Persistence} (c). When $T \subset \RR^{q}$, the topology of $\mathfrak{C}(X, a)$ is homotopy equivalent to $X_a$ \citep{hatcher2002algebraic} (p. 257). Consequently, the persistence diagrams with Cech complex filtration equals to the persistence diagrams with ball model filtration.

\paragraph{Vietoris-Rips complex (a.k.a. Rips complex) filtration.} Given a set $X = \left\{ x_1, x_2, ..., x_m \right\}$ in a metric space $(T, d_T)$. For $a \ge 0$, we form a $t$-simplex from a $(t + 1)$-point subset $X_{t+1}$ of $X$ which satisfies $d_T(x, z)\le 2a, \forall x, z \in X_{t+1}$. The set of all these simplices is called Vietoris-Rips complex of $X$ with parameter $a \ge 0$, denoted as $\mathfrak{R}(X, a)$. For $a < 0$, we define $\mathfrak{R}(X, a) = \emptyset$. Therefore, $\left\{ \mathfrak{R}(X, a) \mid a \in \RR \right\}$ can be used as a filtration as illustrated in Figure \ref{fg:Persistence} (d).

\paragraph{Sub-level set filtration.}
Let $T$ be a topological space, given a function $\mathfrak{f} : T \to \RR$ as an input, and defined a sub-level set $\mathfrak{F}_a := \mathfrak{f}^{-1}\left( ( - \infty, a ]\right)$. Thus, $\left\{ \mathfrak{F}_a \mid a \in \RR \right\}$ can be regarded as a filtration as in Figure \ref{fg:Persistence} (g). For example, it is easy to see that a connected component has a birth at $\mathfrak{F}_p$ and it is death at $\mathfrak{F}_q$ as in Figure \ref{fg:Persistence} (g). Thus, a point $(p, q)$ is in persistence diagrams of $0$-dimensional topological feature for the given function $\mathfrak{f}$. In Figure \ref{fg:Persistence} (g), persistence diagram of $0$-dimensional topological feature for $\mathfrak{f}$ is $\Dg = \left\{ (p, q); (t, \infty)\right\}$.

%%%%%%%%%%%%%%%%%%%%
\section{Kernels}\label{subsec:Kernel}

We review some important definitions and theorems about kernels.

\paragraph{Positive definite kernels.} A function $k: X \times X \rightarrow \RR$ is called a positive definite kernel if $\forall n \in \NN^{*}, \forall x_1, x_2, ..., x_n \in X$, $\sum_{i, j} c_i c_j k(x_i, x_j) \ge 0$, $\forall c_i \in \RR$.

\paragraph{Negative definite kernels.} A function $k: X \times X \rightarrow \RR$ is called a negative definite kernel if $\forall n \in \NN^{*}, \forall x_1, x_2, ..., x_n \in X$, $\sum_{i, j} c_i c_j k(x_i, x_j) \le 0$, $\forall c_i \in \RR$ such that $\sum_i c_i = 0$.

%\begin{equation}
%k_{t}(x, z) := \exp{\left(- t \kappa(x, z)\right)}
%\end{equation}

\paragraph{Berg-Christensen-Ressel Theorem.}
In \citep{berg1984harmonic} (Theorem 3.2.2, p.74), if $\kappa$ is a \textit{negative definite} kernel, then $k_{t}(x, z) := \exp{\left(- t \kappa(x, z)\right)}$ is a positive definite kernel for all $t > 0$. For example, Gaussian kernel $k_{t}(x, z) = \exp{\left(- t \norm{x - z}^2_2\right)}$ is positive definite since it is easy to check that squared Euclidean distance is indeed a negative definite kernel\footnote{$\forall n \in \NN^{*}, \forall x_1, x_2, ..., x_n \in X$, and $\forall c_i \in \RR$ such that $\sum_i c_i = 0$, we have $\sum_{i, j} c_i c_j \norm{x_i - x_j}^2_2 = \sum_i c_i x_i^2 \sum_j c_j + \sum_i c_i \sum_j c_j x_j^2 - 2\sum_{i, j} c_i c_j x_i x_j = - 2\left(\sum_i c_i x_i \right)^2 \le 0$ .}.

\paragraph{Schoenberg Theorem.} In \citep{schoenberg1942positive} (Theorem 2, p. 102), a function $f(\left<\cdot, \cdot\right>)$ defined on the unit sphere in a Hilbert space is positive definite if and only if its Taylor series expansion has only nonnegative coefficients,
\begin{equation}
f(\xi) = \sum_{i=0}^{\infty} a_i \xi^{i}, \quad \text{with} \, a_i \ge 0.  
\end{equation}

%%%%%%%%%%%%%%%%%%%%
\section{Related Kernels for Persistence Diagrams}\label{sec:Related}

\paragraph{Persistence Scale Space kernel ($k_{\PSS}$).} \cite{reininghaus2015stable} proposed the Persistence Scale Space (PSS) kernel, motivated by a heat diffusion problem with a Dirichlet boundary condition. The PSS kernel between two PDs $\Dg_i$ and $\Dg_j$ is defined as $k_{\PSS}\left( \Dg_i, \Dg_j \right) :=  \frac{1}{8 \pi \sigma} \sum_{\substack{p_i \in \Dg_i \\ p_j \in \Dg_j}} \exp\left( - \frac{\norm{p_i - p_j}_2^2}{8\sigma} \right) - \exp\left( - \frac{\norm{p_i - \bar{p}_{j}}_2^2}{8\sigma} \right)$, where $\sigma$ is a scale parameter and if $p = (a, b)$, then $\bar{p} = (b, a)$, mirrored at the diagonal $\Delta$. The time complexity is $O(N^2)$ where $N$ is the bounded cardinality of PDs. By using the Fast Gauss Transform \citep{greengard1991fast} for approximation with bounded error, the time complexity can be reduced to $O(N)$.

\paragraph{Persistence Weighted Gaussian kernel ($k_{\PWG}$).} \cite{kusano2016persistence} proposed the Persistence Weighted Gaussian (PWG) kernel by using kernel embedding into the reproducing kernel Hilbert space. Let $k_{G_{\sigma}}$ be the Gaussian kernel with a positive parameter $\sigma$, and associated reproducing kernel Hilbert space $\Hh_{\sigma}$. Let $\mu_i := \sum_{p \in \Dg_i} \arctan \left(C \pers(p)^q \right) k_{G_{\sigma}}(\cdot, p) \in \Hh_{\sigma}$, where $C, q$ are positive parameter, and for $p = (a, b)$, a persistence of $p$ is that $\pers(p) := b - a$. Let $\mu_j$ be defined similarly for $\Dg_j$. Given a parameter $\tau > 0$, the persistence weighted Gaussian kernel is defined as $k_{\PWG}(\Dg_i, \Dg_j) := \exp\left( - \frac{\norm{\mu_i - \mu_j}^2_{\Hh_{\sigma}}}{2\tau^2} \right)$. The time complexity is $O(N^2)$. Furthermore, \cite{kusano2016persistence} also proposed to use the random Fourier features \citep{rahimi2008random} for computing the Gram matrix of $m$ persistent diagrams with $O(m N u + m^2 u)$ complexity, where $u$ is the number of random variables using for random Fourier features. Thus, the time complexity can be reduced to be linear in $N$. 

\paragraph{Sliced Wasserstein kernel ($k_{\SW}$).} \cite{carriere17asliced} proposed the Sliced Wasserstein (SW) kernel, motivated from Wasserstein geometry for PDs. However, it is well-known that the Wasserstein distance is not negative definite. Therefore, it may be necessary to \textit{approximate} the Wasserstein distance to design positive definite kernels on Wasserstein geometry for PDs. Indeed, \cite{carriere17asliced} use the SW distance, which is an approximation of Wasserstein distance, for proposing the positive definite SW kernel, defined as $k_{\SW}(\Dg_i, \Dg_j) := \exp\left( - \frac{d_{\SW}(\Dg_i, \Dg_j)}{2\sigma^2}\right)$. The time complexity for the SW distance $d_{\SW}(\Dg_i, \Dg_j)$ is $O(N^2\log N)$, and for its $M$-projection approximation, it is $O(M N \log N)$.

\paragraph{Metric preservation.} For those kernel methods for PDs, only the SW kernel preserves the metric between PDs, that is the Wasserstein geometry. Furthermore, \cite{carriere17asliced} argued that this property should lead to improve the classification power. In this work, we explore an alternative Riemannian manifold geometry for PDs, namely the Fisher information metric which is also known as a particular pull-back metric on Riemannian manifold \citep{le2015unsupervised}. Moreover, the proposed positive definite Persistence Fisher kernel is directly built upon the Fisher information metric for PDs \textit{without approximation} while it may be necessary to approximate the Wasserstein distance for designing positive definite kernels on Wasserstein geometry for PDs. Additionally, the time complexity of the Persistence Fisher kernel is also better than the Sliced Wasserstein kernel in term of computation.

%%%%%%%%%%%%%%%%%%%%
\section{More Experiments on Hemoglobin Classification}

We evaluated the Persistence Fisher kernel on Hemoglobin classification for the \textit{taunt} and \textit{relaxed} forms \citep{cang2015topological}. For each form, there are $9$ data points, collected by the X-ray crystallography. As in \citep{kusano2017kernelJMLR}, we selected $1$ data point from each class for test and used the rest for training. There are totally $81$ runs. We also compared with the molecular topological fingerprint (MTF) for SVM \citep{cang2015topological}. We summarize averaged accuracy results on SVM in Table \ref{tb:Hemoglobin}. The Persistence Fisher kernel again outperformances other baseline kernels, and also SVM with MTF.

\begin{table*}[]
\centering
\caption{Averaged accuracy results (\%) on SVM classification. The result of MTF with SVM is cited from \citep{cang2015topological}.}
\label{tb:Hemoglobin}
\begin{tabular}{|c|c|c|c|c|c|c|}
\hline
                           & MTF & $k_{\PSS}$ & $k_{\PWG}$ & $k_{\SW}$ & Prob+$k_G$ & \textbf{\boldmath{$k_{\PF}$}} \\ \hline
Accuracy ($\%$) & $84.50$ & $83.33$    & $88.89$    & $88.89$   & $83.95$         & \textbf{97.53}   \\ \hline
\end{tabular}
\end{table*}

\bibliographystyle{plainnat}
\bibliography{MLREF}

\end{document}